\documentclass{article}
\usepackage{geometry}
\usepackage{indentfirst} 

\usepackage{amsmath}
\usepackage{amsthm}
\usepackage{amsfonts}
\usepackage{bbm}
\usepackage{amssymb}
\usepackage{dsfont}

\usepackage{newtxtext}
\usepackage{newtxmath}

\usepackage{float}
\usepackage{graphicx}
\usepackage[section]{placeins}
\usepackage{comment}
\usepackage{url}
\usepackage{cleveref}
\usepackage[numbers]{natbib}

\DeclareMathOperator{\GW}{GW}
\DeclareMathOperator{\fGW}{fGW}
\DeclareMathOperator{\Cpl}{Cpl}

\DeclareMathOperator{\supp}{supp}

\DeclareMathOperator{\diam}{diam}

\newcommand{\lb}{\llbracket}
\newcommand{\rb}{\rrbracket}
\newcommand{\V}{\mathbb{V}}
\newcommand{\M}{\mathbb{M}_n}
\newcommand{\Mleq}{\mathbb{M}_n^{\leq}}
\newcommand{\X}{\mathbb{X}}
\newcommand{\Y}{\mathbb{Y}_f}
\newcommand{\R}{\mathbb{R}}
\newcommand{\N}{\mathbb{N}}
\newcommand{\Prob}{\mathbf{P}}
\newcommand{\Ex}{\mathbf{E}}

\newtheorem{Thm}{Theorem}[section]
\newtheorem{Lem}[Thm]{Lemma}
\theoremstyle{definition}
\newtheorem{Def}[Thm]{Definition}

\title{$k$-Nearest Neighbors in Gromov--Wasserstein Space}
\author{Kaitlyn Hohmeier $^{a}$ \footnote{Supported by NSF Graduate Research Fellowship Program, Grant No. DGE-2040435.} , Nicolas Fraiman $^b$, Caroline Moosm\"uller $^{a}$ \footnote{Supported by NSF DMS-2410140.} \\
\footnotesize \textit{$^a$ University of North Carolina at Chapel Hill, Department of Mathematics}\\ \footnotesize \textit{$^b$ University of North Carolina at Chapel Hill, Department of Statistics and Operations Research}}
\date{}

\begin{document}

\maketitle

\begin{abstract}
    The Gromov--Wasserstein (GW) distance provides a framework for comparing metric measure spaces, regardless of their underlying structure or geometry. For network-based data, it enables direct comparisons of graphs with different numbers of nodes, without requiring an embedding or other abstraction. Furthermore, through a variant of GW known as fused Gromov--Wasserstein (fGW), it is also possible to incorporate node features in addition to graph structure. In this work, we implement $k$-nearest neighbors ($k$-NN) classification using the GW and fGW distances. We prove the universal consistency of the GW-$k$-NN classifier on the space of equivalence classes of metric measure spaces with finite support and uniform probability measure. By viewing graphs as finitely supported metric measure spaces equipped with the pairwise distance metric and a uniform probability measure on the nodes, we obtain universal consistency of GW-$k$-NN for the space of graphs. Likewise for fGW-$k$-NN, we prove universal consistency on the space of weak isomorphism classes of structured objects consisting of metric measure spaces with finite support and uniform probability measure and feature maps into Euclidean space, thus establishing universal consistency on the space of node-attributed graphs. Our numerical experiments show that GW-$k$-NN and fGW-$k$-NN consistently perform well across multiple graph datasets, suggesting that metric classifiers such as $k$-NN work well in the GW framework. 
\end{abstract}

\section{Introduction}
The technique of $k$-nearest neighbors ($k$-NN) is a well-established machine learning classifier: given a test object $X$, one computes distances to all elements in the training data $\{(X_i, Y_i)\}_{i=1}^n$ whose labels are known and assigns $X$ a label $Y$ by majority vote among the $k$ closest neighbors. The $k$-NN classifier can be equipped with many types of distance functions and implemented in both Euclidean and non-Euclidean metric spaces; our metric spaces of interest are the Gromov--Wasserstein space and the fused Gromov--Wasserstein space. M\'emoli first introduced the Gromov--Wasserstein framework of optimal transport in 2011 as an extension and refinement of the Gromov--Hausdorff distance, creating the \textit{Gromov--Wasserstein distance} as a distance between metric measure spaces \cite{gwdist}.

The Gromov--Wasserstein framework of optimal transport and its variants have many applications in machine learning and classification tasks that involve comparison or alignment of structured, complex data. For example, the Gromov--Wasserstein distance has been used in shape matching as shown in work by \citet{discretegw}. Prior work by \citet{entropicgromov} illustrated the use of the entropic-regularized Gromov--Wasserstein distance in image comparison and alignment. The framework has also been used in biological applications, such as by \citet{biologygw} who created a Gromov--Wasserstein-like distance between biological time series data. Gromov--Wasserstein-based distances have also been applied in graph classification; the fused Gromov--Wasserstein distance, for example, was used for this purpose with graph kernel methods in \citet{vay2019fgw}.

\subsection{Notation and Definitions}\label{introinfo}
Throughout, a \emph{metric measure space} (or \emph{mm--space}) is a triple $\mathcal{X}=(X,d_X,\mu_X)$, where $(X,d_X)$ is a metric space and $\mu_X$ is a Borel probability measure on $X$. We write $\mathrm{supp}(\mu_X)$ for its support and $\phi_\#\mu_X$ for the pushforward of $\mu_X$ by a measurable map $\phi$.
Given probability measures $\mu_X$ on $X$ and $\mu_Y$ on $Y$, we denote by $\Cpl(\mu_X,\mu_Y)$ the set of couplings on $X\times Y$ with marginals $\mu_X$ and $\mu_Y$.
For an mm--space $\mathcal{X}$ we denote by $\lb\mathcal{X}\rb$ its equivalence class under measure-preserving isometries, i.e.\ $\lb\mathcal{X}\rb = \{\mathcal{Y} : \GW(\mathcal{X},\mathcal{Y}) = 0\}$.
When specializing to graphs, we view a graph as an mm--space with finite support and the uniform measure $\mu_X = \frac{1}{|X|}\sum_{x\in X}\delta_x$ on the node set.

\begin{Def}[Gromov--Wasserstein, Continuous Case]\label{gwdistgeneral}
    Denote two mm--spaces by $\mathcal{X}=(X, d_X, \mu_X)$ and $\mathcal{Y}=(Y, d_Y, \mu_Y)$, where $d_X$ and $d_Y$ are metrics and $\mu_X$ and $\mu_Y$ are probability measures. Then
    \[
    \GW_{p}(\mathcal{X}, \mathcal{Y}) =
    \min_{\pi \in \Cpl(\mu_X, \mu_Y)} \left( \int_{(X \times Y)^2} |d_X(x,x') - d_Y(y,y')|^p \ d \pi(x,y) d\pi(x', y') \right)^{1/p},
    \]
    where $\Cpl(\mu_X, \mu_Y)$ is the set of probability measures on $X \times Y$ with marginals $\mu_X$ and $\mu_Y$ and $p \in [1, \infty)$.
\end{Def}

The main drawback of the Gromov--Wasserstein distance is computational complexity: the optimization problem is a non-convex problem that in full generality is NP-hard. For example, in the setting of equivalence classes of finite mm--spaces, an approximation solved using gradient descent has $O(n^3 \log n)$ complexity, where $n$ is the cardinality of the underlying mm--spaces, as shown in \citet{GWBary} and \citet{BBS}. Now consider an mm--space $\mathcal{X} = (X, d, \mu)$. M\'emoli's work introducing the Gromov--Wasserstein distance showed that it defines a metric on the space of mm--spaces up to some choice of isomorphism \cite{gwdist}, defining a metric on the set of equivalence classes of mm--spaces, where
\[
\lb\mathcal{X}\rb = \{ \mathcal{Y} : \GW(\mathcal{X}, \mathcal{Y}) = 0 \}.
\] 
We have that $\GW(\mathcal{X}, \mathcal{Y}) = 0$ if and only if there exists a bijective isometry $\phi : X \rightarrow Y$ such that $\phi_\#$ is measure-preserving, that is, $\phi_\# \mu_X = \mu_Y$. Letting $\mathbb{Z}$ denote a collection of equivalence classes of mm--spaces, we can define the $p$-th Gromov--Wasserstein space as $ (\mathbb{Z}, \GW_p)$. In full generality, this space is not complete (see \citet{gwdist} and \citet{BBS}). We will denote our set of equivalence classes of interest in the following way.
\begin{Def}
    $\X$ denotes the set of the equivalence classes of mm--spaces with finite support and uniform probability measure.
\end{Def}
For graphs, the Gromov--Wasserstein distance applies in the following way: view each graph as a finitely-supported mm--space, equipped with the pairwise distance metric and a uniform probability measure on the nodes. We represent the corresponding Gromov-Wasserstein space, with $p=2$, as $(\X, \GW_2)$, where we are interested in the $p=2$ case because this is what is implemented numerically for GW  computations; see \citet{flamary2021pot}.

A Gromov--Wasserstein variant, called \textit{fused Gromov--Wasserstein}, defines a metric on the space of \textit{structured objects}, uniting the feature-based approach of Wasserstein optimal transport with the structure-based approach of Gromov--Wasserstein optimal transport to study objects that contain both feature and geometric information. This framework was first introduced in 2019 by \citet{vay2019fgw}. Formally, a structured object is defined as follows.

\begin{Def}
    A structured object over a feature space $(\Omega, d)$ is a pair $(\mathcal{X}, f)$, where:
    \begin{itemize}
        \item $\mathcal{X} = (X, d_X, \mu_X)$ is the structure space equipped with metric  $d_X : X \times X \rightarrow \R_+$;
        \item $(\Omega, d)$ is the feature space, equipped with the metric $d: \Omega \times \Omega \rightarrow \R_+$; and
        \item Features are assigned to each element of $X$ via \textit{feature maps}
$f : X \to \Omega$.
\end{itemize}
\end{Def}

We can then define the fused Gromov--Wasserstein distance in the following way. Here, we incorporate Definition 8 from \citet{vay2019fgw} with Definition 3.8 from \citet{Z-GW}.
\begin{Def}[Fused Gromov--Wasserstein, Continuous Case]\label{def:fgwdist}
Consider $\mathcal{X} = (X, d_X, \mu_X)$ and $\mathcal{Y} = (Y, d_Y, \mu_Y)$ with feature maps
$f : X \to \Omega$ and $g : Y \to \Omega$. Let $p,q \geq 1$ and $\alpha \in (0,1)$. Then 
    \begin{align*}
    &\fGW_{\alpha, p,q}((\mathcal{X},f), (\mathcal{Y}, g)) \\
    &\qquad = \min_{\pi\in\Cpl(\mu_X,\mu_Y)} \left(\int_{(X \times Y)^2} \left[
        \alpha \, |d_X(x,x') - d_Y(y,y')|^q  + (1-\alpha) \, d(f(x), g(y))^q 
      \right]^{p/q} \ d\pi(x,y) d\pi(x',y') \right)^{1/p}
    \end{align*}
where the minimum is over the set of couplings between $\mu_X$ and $\mu_Y$.
\end{Def}
Here, $\alpha$ denotes a ``trade-off'' parameter between structure and feature information. In practice, this variable must be tuned to identify the optimal $\alpha$ value (in classification experiments, for example, this would entail identifying the $\alpha$ value that yields the highest classification accuracy). For graphs, the fused Gromov--Wasserstein distance applies to graphs with node attributes; see Figure \ref{fig:graphexs} as an illustration.

\begin{figure}[H]
    \centering
    \includegraphics[scale=0.25]{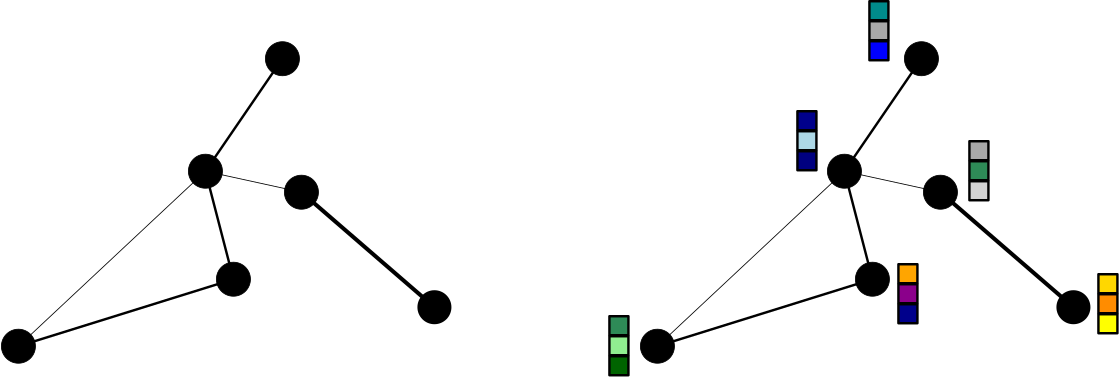}
    \caption{Comparison of a graph versus an attributed graph. Left: By viewing a graph as an mm--space where the metric is pairwise distance between nodes and equipped with a uniform probability measure, we can apply the Gromov--Wasserstein framework to graphs. Line thickness denotes edge weights, which affect pairwise distances. Equal-sized nodes denote uniform probability measure on the nodes. Right: Node-attributed graphs (here, node attributes are vectors in $\R^3$) can be used with the fused Gromov--Wasserstein framework, with both structure and feature information being taken into consideration. In this example, the feature space is $(\R^3, d)$ where $d$ is the standard Euclidean metric.}
    \label{fig:graphexs}
\end{figure}

In an analogous sense to Gromov--Wasserstein optimal transport, the fused Gromov--Wasserstein distance defines a metric on the set of equivalence classes of the space of structured objects. We must therefore define what it means for two structured spaces to be equivalent: \textit{weak isomorphism}. Here, we incorporate the definition of weak isomorphism from Definition 28 in \citet{Z-GW} with the notion of structured objects from \citet{vay2019fgw}. 
\begin{Def}\label{def:strongisom}
    Two structured objects $(\mathcal{X}, f)$ and $(\mathcal{Y}, g)$, where $\mathcal{X} = (X, d_X, \mu_X)$ and $\mathcal{Y} = (Y, d_Y, \mu_Y)$, are \textit{weakly isomorphic} if there exists a structured object $(\mathcal{W}, h)$, where $\mathcal{W} = (W, d_W, \mu_W)$, and measure-preserving maps $\phi_X : W \to X$ and $\phi_Y : W \to Y$ such that $\mu_W \otimes \mu_W$-almost everywhere:
    \[d_X(\phi_X(w), \phi_X(w')) = d_Y(\phi_Y(w), \phi_Y(w')) = d_W(w,w')\]
    and $\mu_W$-almost everywhere:
    \[f(\phi_X(w)) = g(\phi_Y(w)) = h(w).\]
\end{Def}
 A weak isomorphism between structured objects defines an equivalence relation over the space of structured objects. Hence, we can also apply $k$-NN to the metric space of equivalence classes of structured objects equipped with the fused Gromov--Wasserstein distance, applicable to graphs with node attributes. In particular, the use of the $p/q$ power in \Cref{def:fgwdist} ensures that it is well-defined as a metric for all choices of $q$ on the space of equivalence classes of structured objects; see Corollary 4.3 in \citet{Z-GW}. We denote our collection of  isomorphism classes of interest as $\Y$.

\begin{Def}\label{def:Xatt}
     $\Y$ denotes the set of isomorphism classes of structured objects $(\mathcal{X},f)$, where $\mathcal{X}$ is a mm--space with finite support and uniform probability measure and $f$ is a feature map into $\R^d$.
\end{Def}
We represent the corresponding fused Gromov--Wasserstein space, with $p=1$ and $q=2$, as $ (\Y, \fGW_{\alpha, 1,2})$. We are interested in the case $p=1$ and $q=2$ because this is the case implemented numerically for fGW computations; see \citet{vay2019fgw} and \citet{flamary2021pot}.

Following the setup of the binary classification problem shown in \citet{classificationprov}, let $\mathbb P$ be a probability distribution over $W \times\{0,1\}$, where $Y\in\{0,1\}$ denotes the class label and $(W,d_W)$ is a metric space. Hence, ordered pairs of the form $(X, Y)$ are random variables taken from $W$ and $\{0,1\}$. Then a function $g : W \rightarrow \{0,1\}$ is a \textit{classifier}, whose error probability or risk function is given by $L(g) = \Prob(g(X)\neq Y)$. The goal of a classification problem is to minimize this risk. The Bayes classifier, given by $g^*(X)=\mathbf 1_{\{\eta(X)\ge 1/2\}}$, with $\eta(X)=\mathbb P(Y=1\mid X)$ denoting the conditional probability that $Y$ is 1 given some observation $X$, minimizes this risk. This minimum risk is called the \textit{Bayes risk} and is denoted by $L^* = \Prob(g^*(X) \neq Y)$. The Bayes risk $L^*$ represents the theoretical best performance of a classification function \cite{classificationprov}. An important property of any classifier is that of \textit{consistency}. Suppose we have a sequence of training data $D_n = ((X_1, Y_1), \dots, (X_n, Y_n))$. By constructing a sequence of classification functions $\{g_n\}$ (known as a classification rule), such that $L_n = L(g_n) = \Prob(g_n(X, D_n) \neq Y \mid D_n)$ converges in probability to $L^*$, we obtain the idea of \textit{consistency}. A classification rule is \textit{weakly consistent} if $\Ex(L_n) \rightarrow L^*$ as $n \rightarrow \infty$ and \textit{strongly consistent} if $\lim_{n\rightarrow \infty} L_n = L^*$ with probability 1.

However, since in general $\Prob$ is unknown, we want to ensure that consistency is \textit{universal}---that is, regardless of the choice of distribution. Universal consistency is obtained if consistency holds for any choice of distribution. More precisely, consider a finite independent and identically distributed sample of training data $D_n = ((X_1, Y_1), \dots, (X_n, Y_n))$ drawn from $\Prob$. Supervised learning uses a \textit{learning rule} $h_n : (W \times \{0,1\})^n \times W\rightarrow \{0,1\}$, giving the classification rule $g_n = h_n(D_n)$ with risk $L_n = \Prob(g_n(X) \neq Y \mid D_n)$. In this setting, $g_n$ can be determined to be \textit{universally weakly consistent} or \textit{universally strongly consistent}.
\begin{Def}
    Let $L_n$ be as defined in the paragraph above.
    \begin{enumerate}
        \item If $\lim_{n \rightarrow \infty} \Ex(L_n) = L^*$, $g_n$ is universally weakly consistent.
        \item If $\lim_{n\rightarrow \infty} L_n = L^*$ with probability 1 for all distributions $\Prob$, $g_n$ is universally strongly consistent.
    \end{enumerate}
\end{Def}

In this paper, we focus on \textit{graph classification}: that is, given a collection of observed graph data $\{(X_i, Y_i)\}_{i=1}^{n}$, with $X_i$ a graph and $Y_i$ a label, determine the label $Y$ associated with a new graph $X$. 
 We consider both binary and multiclass classification problems in the metric space $(\X, \GW_2)$ and binary classification problems in the metric space $(\Y, \fGW_{\alpha, 1, 2})$ using the $k$-NN classifier. Our main contribution, establishing the universal consistency of $k$-NN in $(\X, \GW_2)$ and in $(\Y, \fGW_{\alpha, 1, 2})$, suggests that metric classifiers such as $k$-NN work well in the Gromov--Wasserstein framework with network data. In particular, this approach offers a highly interpretable framework for network classification: since the Gromov--Wasserstein or fused Gromov--Wasserstein distance between two graphs will be 0 if and only if the graphs are isomorphic, we can therefore understand ``small'' GW-based distances between graphs, that is, graphs that are ``close'' together, as a measure of their structural similarity. We will also showcase numerical results for both methods, illustrating that they work well for classifying both non-attributed and attributed graphs.

\subsection{Main Results}
Our main theoretical results establish the universal consistency of $k$-NN on $(\X, \GW_2)$ and $(\Y, \fGW_{\alpha, 1, 2})$.

\begin{Thm}\label{consistency}
    Let $\X$ denote the set of equivalence classes of mm--spaces $(M,d_M, \mu_M)$ with finite support and uniform probability measure $\mu_M$. Then $\X$ has $\sigma$-finite metric dimension, and under the assumptions that $k \to \infty$ and $k/n \to 0$, the $k_n$-NN classifier is universally weakly consistent on $(\X, \GW_2)$. If in addition $k/\log(n) \to \infty$, then the $k_n$-NN classifier is universally strongly consistent on $(\X, \GW_2)$.
\end{Thm}

\begin{Thm}\label{consistencyFGW}
    Let $\Y$ denote the set of weak isomorphism classes of structured objects $(\mathcal{X}, f)$, where each $\mathcal{X} = (X, d_X, \mu_X)$ is a finitely supported mm--space with uniform probability measure and each $f$ is a feature map $f : X \to \R^d$. Let $\alpha \in (0,1)$. Then $\Y$ has $\sigma$-finite metric dimension, and under the assumptions that $k \to \infty$ and $k/n \to 0$, the $k_n$-NN classifier is universally weakly consistent on $(\Y, \fGW_{\alpha, 1, 2})$. If in addition $k/\log(n) \to \infty$, then the $k_n$-NN classifier is universally strongly consistent on $(\Y, \fGW_{\alpha, 1, 2})$.
\end{Thm}

In particular, these results establish the universal consistency of the GW-$k_n$-NN and the fGW-$k_n$-NN classifiers on the spaces of graphs and node-attributed graphs, respectively.

\subsection{Related Work}\label{consistencyhistory} 
A broad generalization of GW-based distances is given by the \textit{Z-Gromov-Wasserstein} distance. Introduced by \citet{Z-GW}, the Z-Gromov--Wasserstein framework unites all frameworks of optimal transport via the choice of a metric space $(Z, d_Z)$, where the choice of metric space dictates which optimal transport setting one retrieves. \cite{Z-GW} For example, choosing $Z = \R$, equipped with the standard metric, retrieves the Gromov--Wasserstein framework, while choosing $Z = \R \times Y$, where $Y$ is any metric space, yields fused Gromov--Wasserstein. 

In the area of graph classification, there are some numerical results involving graph classification with optimal transport methods. For example, prior work in \citet{vay2019fgw} on fused Gromov--Wasserstein optimal transport showcased numerical experiments involving a support vector machine (SVM) with the indefinite kernel matrix $e^{-\gamma FGW}$, where $FGW$ denotes the fused Gromov--Wasserstein distance, compared with other graph kernel methods. This work also included numerical experiments with this SVM method using the Gromov--Wasserstein distance on two non-attributed social graph datasets, the first such application of the Gromov--Wasserstein distance. The closest example we can find in prior work to our classification task of interest with $k$-NN is the use of the Gromov--Wasserstein distance with $k$-NN in shape classification (see \citet{2508.02364}); however, to the best of our knowledge, there do not exist any applications of GW-$k$-NN in graph classification, nor do there exist theoretical results for GW-$k$-NN for graph classification, or for any other mm--space setting. Likewise, we are unaware of any other such applications and theoretical results for fGW-$k$-NN.

In Euclidean settings, the theory of $k$-NN is well-established and well-understood, beginning with Cover and Hart \cite{coverhart} in 1967 showing that the $1$-NN error is bounded by twice the Bayes error, and extended by Stone \cite{stone} in 1977 to prove universal consistency. The $k$-NN classification method continues to be an active area of research. For example, in more recent work, \citet{dgw} in 2018 generalized previous Bayes error rate results for $k$-NN and proved a rate of convergence for the $L_2$ error, while \citet{cfs} in 2020 derived asymptotic results under noisy training data. For universal consistency, \citet{stone} established that $k$-NN is weakly universally consistent  under the conditions (1) $k \rightarrow \infty$ and (2) $k/n \rightarrow 0$, where $k$ grows with the sample size $n$, with universal strong consistency holding if in addition (3) $k / \log n \rightarrow \infty$, as shown by \citet{strongconsistency}.

Within optimal transport theory, a 2023 paper by Ponnoprat \cite{ponnoprat} established both negative and positive $k$-NN consistency results in the Wasserstein space. The major results of this paper established that $k$-NN is not universally consistent on $\mathcal{W}_p((0,1))$ and proved universal consistency in certain special cases. Of particular interest to us is one special case, in which it was shown that $k$-NN is universally consistent on the space of finitely-supported measures equipped with the $1$-Wasserstein distance (see Theorem 3 in \cite{ponnoprat}). In this case, universal consistency was proven using a property of certain metric spaces known as \textit{$\sigma$-finite dimension}. Using $\sigma$-finite dimension as a means to prove universal consistency comes from a combination of work by \citet{Preiss}, \citet{CerouGuyader2006}, \citet{NIPS2014_2b764b80}, and \citet{assouad}, whose results collectively established that $k$-NN is universally consistent on separable metric spaces that have the property of $\sigma$-finite dimension. In a general metric space $(X,d)$, conditions (1), (2), and (3) are not sufficient to guarantee universal consistency of a $k$-NN classifier, as shown by \citet{kumari}. Work by Chaudhuri and Dasgupta (see Theorem 1 in \cite{NIPS2014_2b764b80}) established the additional condition needed for universal strong consistency to hold. In addition to conditions (1), (2), and (3), strong universal consistency holds on a separable metric space $(X, d)$ if the \textit{Lebesgue differentiation condition} is true on that metric space: for any bounded, $\mu$-measurable function $f$, with $\mu$ a Borel probability measure, we have
\[
\lim_{r \downarrow 0} \frac{1}{\mu(B(x,r))} \int_{B(x,r)} f \ d \mu = f(x) \quad \mu-a.e.,
\] 
where $B(x,r)$ denotes a ball in $(X,d)$.  Ultimately, it was observed by  \citet{CerouGuyader2006} and subsequently proven by \citet{assouad} that this differentiability condition is equivalent to $\sigma$-finite dimension, a proof originally outlined by Preiss \cite{Preiss}. 
Hence, to summarize this prior work, $k$-NN will be universally consistent on a separable metric space that exhibits the $\sigma$-finite dimension property. Furthermore, to underscore that $k$ grows with $n$, we will refer to this classifier as $k_n$-NN.

\subsection{Outline}
In Section \ref{sec:dimensions}, we prove a result about a particular notion of dimension, which we will use to prove our main results. Our universal consistency proof strategy involves proving that $(\X, \GW_2)$ exhibits \textit{$\sigma$-finite dimension}, which we discuss in more detail in Section \ref{sec:GWuc}. In Section \ref{sec:fgw-consistency}, we offer a model of $n$-point attributed spaces, and like in the GW case, we ultimately establish universal consistency in $(\Y, \fGW_{\alpha, 1, 2})$ by showing that this space also has $\sigma$-finite dimension. Finally, in \Cref{sec:numericalresults}, we present the results of numerical experiments with GW-$k_n$-NN and fGW-$k_n$-NN applied to several benchmark graph datasets.

\section{Finite Assouad-Nagata Dimension under Bi-Lipschitz Maps}\label{sec:dimensions}

In this section, we present and prove a result that shows that a bi-Lipschitz map will prove the finiteness of Assouad-Nagata dimension. Ultimately, we will use that result, \Cref{lem:packingmetricdim}, as the main tool to prove universal consistency in \Cref{sec:GWuc} and \Cref{sec:fgw-consistency}. 
We first define the concept of 
\textit{Nagata dimension} (also known as \textit{Assouad-Nagata dimension}). First, following the definitions in \citet{Lang}, for a metric space $(X,d)$ and $Y = \{ Y_i\}_{i \in I}$ a family of subsets of $X$, $Y$ is \textit{D-bounded} for some constant $D \geq 0$ if $\mathrm{diam} (Y_i) := \sup \{d(x,x') : x, x' \in Y_i \} \leq D$ for all $i \in I$. Then the Assouad-Nagata dimension is defined in the following way.
\begin{Def}
    For $s > 0$, the \textit{s-multiplicity} of $Y$ is the infimum of all such $\beta$ such that every subset of $X$ with diameter $\leq s$ meets at most $\beta$ members of the family. 
\end{Def}

\begin{Def}\label{ANdim}
    The Assouad-Nagata dimension of $X$, $\dim_{AN}(X)$, is the infimum of all integers $\beta$ with the following property: There exists a constant $c > 0$ such that for all $s > 0$, $X$ has a $cs$-bounded covering with $s$-multiplicity at most $\beta + 1$ (with $\dim_{AN}(X) = \infty$ if no $\beta$ exists). For a subset $Y$ of $X$, we say that $Y$ has Assouad-Nagata dimension $\delta \in \N$ on scale $s$ inside of $X$ if and only if for some $c > 0$ and integer $\beta$, there exists a covering $\mathcal{U}$ of $Y$ by sets of diameter $\leq cs$ such that every subset of $X$ of diameter $\leq s$ meets at most $\beta$ elements of $\mathcal{U}$. We denote this by $\dim^{s}_{AN}(Y, X)$.
\end{Def}

The notion of Assouad-Nagata dimension was first introduced by \citet{JNagata} and then later adapted by \citet{assouaddimension}. The main interest in this dimension is to determine if it is finite or infinite; however, as motivated in Example 1 and subsequent discussion in \citet{ponnoprat}, the idea of Assouad-Nagata dimension fails for infinite-dimensional spaces. But, it is still useful to define the notion of $\sigma$-finite metric dimension, using Definition 6.6 in \citet{nagatadim3}.

\begin{Def}\label{sigmafinite}
    A metric space $(X,d)$ is said to have \textit{$\sigma$-finite metric dimension} if there exists a countable family $\{Y_n\}_{n \in \N}$ of subsets of $X$ such that every subspace $Y_n$ has finite Assouad-Nagata dimension in $X$ for some scale $s_n > 0$ (where the scales $s_n$ are possibly all different) and $X = \bigcup^\infty_{n=1} Y_n$. We denote this by $\dim_{AN}^{s_n}(Y_n, X)$.
\end{Def}

Hence, a space with $\sigma$-finite metric dimension is one with a countable cover, where each subset in the cover has finite Assouad-Nagata dimension. The notion of $\sigma$-finite metric dimension was first introduced by \citet{Preiss} in connection with the Lebesgue differentiation condition.

With these definitions, we show that if there exists a bi-Lipschitz map between two metric spaces, where the target space is a quotient space of Euclidean space, then we can establish finite Assouad-Nagata dimension in a larger ambient, and in particular we can show that that larger ambient space has $\sigma$-finite metric dimension.
\begin{Lem}\label{lem:packingmetricdim}
     Let $(X, d_X)$ be a metric space with $X = \cup_{n \in \N} W_n$. Suppose that for each $n$, $\phi_n : W_n \to \phi_n(W_n)\subset Z_n$ is bi-Lipschitz with constants $0 < c_n \leq C_n < \infty$, where $Z_n$ is a quotient space of $ \R^{d_n}$ by a finite isometric group action.
     Then $\dim_{AN}^{s}(W_n, X)$ is finite for each $n$, and in particular $(X, d_X)$ has $\sigma$-finite metric dimension.
\end{Lem}

\begin{proof}
    Fix $n$ and write $Z_n=\mathbb R^{d_n}/G_n$, where $G_n$ is a finite group acting isometrically on $\mathbb R^{d_n}$. Let $q_n:\mathbb R^{d_n}\to Z_n$ be the quotient map, and equip $Z_n$ with the quotient metric. Then $q_n$ is $1$--Lipschitz.

    For $s_n>0$, partition $\mathbb R^{d_n}$ into half-open cubes $\{Q_j\}_{j\in J}$ of side length $s_n/\sqrt{d_n}$, and define
    \[
    \mathcal U_{s_n}^{Z_n}:=\{q_n(Q_j):j\in J\}.
    \]
    This is a cover of $Z_n$. Since $q_n$ is $1$--Lipschitz,
    \[
    \diam_{Z_n}(q_n(Q_j))\leq \diam_{\mathbb R^{d_n}}(Q_j)\leq s_n.
    \]

    Define
    \[
    \mathcal U_{s_n}:=\{\phi_n^{-1}(U):U\in\mathcal U_{s_n}^{Z_n}\}.
    \]
    If $V=\phi_n^{-1}(U)$, then the lower bi-Lipschitz bound $c_n d_X(x,x')\leq d_{Z_n}(\phi_n(x),\phi_n(x'))$ gives
    \[
    \diam_{d_X}(V)\leq c_n^{-1}\diam_{Z_n}(U)\leq c_n^{-1}s_n,
    \]
    so $\mathcal U_{s_n}$ is a $c_n^{-1}s_n$--bounded cover of $W_n$.

    It remains to bound the $s_n$--multiplicity. Let $A\subset X$ satisfy $\diam_{d_X}(A)\leq s_n$, and set
    \[
    S:=\phi_n(A\cap W_n)\subset Z_n.
    \]
    If $S=\varnothing$, there is nothing to prove. Otherwise, by the upper bi-Lipschitz bound,
    \[
    \diam_{Z_n}(S)\leq C_n s_n.
    \]
    Choose $z_0\in S$ and a lift $x_0\in q_n^{-1}(z_0)$. For every $z\in S$, the quotient metric gives a lift $\tilde z\in q_n^{-1}(z)$ with
    \[
    \|\tilde z-x_0\|\leq C_n s_n.
    \]
    Hence $S$ has a lift contained in the Euclidean ball $B(x_0,C_n s_n)$, and therefore
    \[
    q_n^{-1}(S)\subset \bigcup_{g\in G_n} B(gx_0,C_n s_n).
    \]
    If $q_n(Q_j)$ meets $S$, then $Q_j$ meets $q_n^{-1}(S)$. The number of cubes of side length $s_n/\sqrt{d_n}$ meeting one ball of radius $C_n s_n$ is bounded by
    \[
    (2C_n\sqrt{d_n}+3)^{d_n}.
    \]
    Therefore the number of elements of $\mathcal U_{s_n}^{Z_n}$ meeting $S$ is at most
    \[
    N_n:=|G_n|(2C_n\sqrt{d_n}+3)^{d_n}<\infty.
    \]
    Thus, every subset $A\subset X$ of diameter at most $s_n$ meets at most $N_n$ elements of $\mathcal U_{s_n}$. Hence
    \[
    \dim_{AN}^{s_n}(W_n,X)\leq N_n-1<\infty.
    \]
    Since $X=\bigcup_{n\in\mathbb N}W_n$, the space $X$ has $\sigma$--finite metric dimension.    
\end{proof}

\section{Universal Consistency of GW-$k_n$-NN on Graphs}\label{sec:GWuc}

In this section, we present and discuss several properties of the space $(\X, \GW_2)$ related to work by Sturm in \cite{Sturm}. In this paper, Sturm does not work with the Gromov--Wasserstein distance directly, but rather with the $L^{p,q}$ distortion distance $\Delta_{pq}$. 
This metric likewise measures distances between equivalence classes of mm-spaces. Hence, it is closely related to the Gromov--Wasserstein distance, and results involving it generalize naturally to the Gromov--Wasserstein setting. In particular, setting $q =1 $ in the $L^{p,q}$ distortion distance exactly recovers the $p$-th Gromov--Wasserstein distance.

We also prove one of our main results in this section, \Cref{consistency}, by showing that $(\X, \GW_2)$ has the property of $\sigma$-finite dimension. As discussed in Section 1, this topological property is related to the Lebesgue differentiation condition, which is needed to prove universal consistency on a general metric space.
Other properties of $(\X, \GW_2)$, such as its connection to a particular space of real-valued symmetric square matrices and \textit{n-point spaces} (both studied in detail in section 5.4 in \citet{Sturm}), will also be proven here. This connection to $n$-point spaces $\X(n)$ is particularly valuable, as it allows us to establish both of our properties of interest: separability and $\sigma$-finite metric dimension, which we need to prove universal consistency as discussed in \Cref{consistencyhistory}.

\subsection{The $n$-Point Spaces Model}

In this section, we show the existence of a bi-Lipschitz map between the $n$-point spaces to the quotient space $\Mleq$ (which we will define shortly) for a fixed $n$. Because we map $\X(n)$ into this quotient space $\Mleq$, we can show that each $n$-point space has finite Assouad-Nagata dimension. Since $\X = \cup_{n\in\N} \X(n)$, $\X$ has $\sigma$-finite metric dimension. Combining this fact with the separability of $\X$ establishes universal consistency.

We first establish the connection between $(\X, \GW_2)$ and $n$-point spaces. Following Sturm \cite{Sturm}, fix $n \in \N$. Define $M_n$ to be the linear space of real-valued symmetric $n \times n$ matrices with zeros on the diagonal, and equip this space with the re-normalized $L_2$-norm

\[
\|f \|_{M_n} = \left(\frac{2}{n^2} \sum_{1 \leq i < j \leq n} f^2_{ij}\right)^{1/2}, \quad f = (f_{ij})_{1 \leq i < j \leq n} \in M_n.
\] 
Because the permutation group $S_n$ acts isometrically on $M_n$ via the map $(\sigma, f) \mapsto \sigma f$, $(\sigma f)_{ij} = f_{\sigma_i \sigma_j}$, we can define an equivalence relation in $M_n$:
\[
f \sim f' \iff \exists \sigma \in S_n : f' = \sigma f
\]
We denote this quotient space by $\M$ and equip it with the metric 
\[
d_{\M}([f],[g]) = \inf \{\|f - \sigma g\|_{M_n} : \sigma \in S_n\}.
\]
This space possesses several important geometric properties, such as completeness, geodesicity, and nonnegative curvature, and is isometric to a cone in $\R^{n(n-1)/2}$ with the induced inner metric in the cone.

A subset of $M_n$, denoted $M_n^\leq$, consists of symmetric $n \times n$ matrices $(f_{ij})_{1\leq i<j\leq n}$ that satisfy a ``triangle inequality'' $f_{ij} + f_{jk} \geq f_{ik}$ for all $i,j,k \in \{1, \dots, n\}$ and where all off-diagonal entries are positive real numbers. We can apply the same equivalence relation $\sim$ on $M_n^\leq$ as on $M_n$, that is:
\[
f, f' \in M_n \textit{ with } f \sim f' : \quad f \in M_n^\leq \iff f' \in M_n^\leq,
\]
creating the quotient space $\Mleq$, the subset of $\M$ of equivalence classes satisfying the ``triangle inequality.'' We again equip $\Mleq$ with the metric $d_{\M}$.

To show the connection between $\X$ and $\Mleq$, we first define the concept of \textit{n-point spaces.}
\begin{Def}\label{def:npointspaces}
    Let $\X(n) \subset \X$ denote the set of isomorphism classes of mm--spaces $(X, d, \mu)$ with $|X| = n$ and $\mu = \frac{1}{n} \sum_{i=1}^n \delta_{x_i}$. These $\X(n)$ spaces are known as \textit{$n$-point spaces}. 
\end{Def}
We show that $\X = \bigcup_{n \in \N} \X(n)$. Let $\lb \mathcal{X} \rb = \lb(X,d,\mu)\rb \in \X$. Then by definition of $\X$, the support of $\mu$ is finite; let $n = |X|$. Then $\lb \mathcal{X} \rb \in \X(n)$; hence, every element of $\X$ lies in some $\X(n)$, establishing $\X \subseteq \bigcup_{n\in\N}\X(n)$. Conversely, for each $n$, $\X(n) \subseteq \X$ by definition. Hence, $\X = \cup_{n \in \N} \X(n)$. 
These $n$-point spaces are connected to $\Mleq$ via a bijective map $\Phi_n : \X(n) \to \Mleq$ that sends a representative metric measure space to its corresponding distance matrix (modulo relabeling). Sturm discusses an injective map related to $\Phi_n^{-1}$ in Section 5.4 in \cite{Sturm}, whose image is mm--spaces with $n$ points and uniform mass. In particular, Proposition 5.25 in \cite{Sturm} establishes that the map is globally 1-Lipschitz for all $n$; that is, for each $n \in \N$, it is 1-Lipschitz, so this bound globally holds for all $n$. For a given fixed $n$, we can further establish that for some constant $C_n$ depending on $n$, we have an upper bound as well, thus giving a bi-Lipschitz map.

\begin{Lem}\label{biLipschitzmap}
    Let $\lb \mathcal{X}\rb, \lb \mathcal{Y} \rb \in \X(n)$ and $\Phi_n : \X(n) \to \Mleq$, and fix $n$. Then $\Phi_n$ is bi-Lipschitz in the following sense: Let $[f] := \Phi_n(\lb \mathcal{X} \rb)$ and $[g] := \Phi_n(\lb  \mathcal{Y}\rb)$. Then we have
    \[c_n \GW_2(\lb \mathcal{X}\rb, \lb \mathcal{Y}\rb) \leq d_{\M}([f], [g]) \leq C_n \GW_2(\lb \mathcal{X}\rb, \lb \mathcal{Y}\rb)\]
    where $c_n = 1$ for all $n$ and $C_n = \sqrt{(n-1)^2 + 1}$.
\end{Lem}

\begin{proof}
    We already have that the lower bound $c_n$ exists and is equal to 1 for all $n$ by Proposition 5.25 in \cite{Sturm}. The proof uses the fact that permutation matrices are a subset of doubly stochastic matrices, so optimizing over the larger Birkhoff polytope of doubly stochastic matrices can only decrease the $\GW_2$ distance, yielding the bound $\GW_2(\lb \mathcal{X}\rb, \lb \mathcal{Y}\rb) \leq d_{\M}([f], [g])$.

    On $\X(n)$, every coupling $\pi$ can be written as $\pi = \frac{1}{n} \sum_{i,j = 1}^n P_{ij} \delta_{x_i, y_j}$ where $P = (P_{ij})$ is a doubly stochastic matrix. Hence, for $\lb \mathcal{X} \rb = (X, d_X, \mu_X), \lb \mathcal{Y} \rb = (Y, d_Y, \mu_Y) \in \X(n)$,
    \[
     \GW_2^2(\lb \mathcal{X} \rb,\lb \mathcal{Y} \rb) = \frac{1}{n^2} \min_{P \in \Pi_n} \sum_{i,j,k,l=1}^n |f_{ik} - g_{jl}|^2 P_{ij} P_{kl}
     \] 
    where $\Pi_n$ denotes the set of doubly stochastic matrices.
    By the Birkhoff-von Neumann theorem, any $P \in \Pi_n$ admits a convex decomposition into permutation matrices:
    \[
    P = \sum_{\sigma \in S_n} \lambda_\sigma \Pi_\sigma, \quad \lambda_\sigma \geq 0, \quad \sum_{\sigma \in S_n} \lambda_\sigma = 1,
    \]
    where $(\Pi_\sigma)_{ij} = \delta_{i \sigma_j}$. Substituting into the GW objective $Q(P,f,g) = \frac{1}{n^2} \sum_{i,j,k,l=1}^n |f_{ik} - g_{jl}|^2 P_{ij} P_{kl}$:
    \[
    Q(P,f,g) = \sum_{\sigma, \tau \in S_n} \lambda_\sigma \lambda_\tau \cdot Q(\Pi_\sigma, \Pi_\tau, f, g),
    \]
    where the cross terms are
    \[
    Q(\Pi_\sigma, \Pi_\tau, f,g) = \frac{1}{n^2} \sum_{i,k} |f_{ik} - g_{\sigma_i \tau_k}|^2 \geq 0.
    \]
    Since for all cross terms we have $Q(\Pi_\sigma, \Pi_\tau, f,g) \geq 0$ and for all weights we have $\lambda_\sigma \lambda_\tau \geq 0$, dropping the off-diagonal terms $\sigma \neq \tau$ only decreases the sum $Q(P, f,g)$. Hence:
    \[
    Q(P,f,g) \geq \sum_{\sigma \in S_n} \lambda_\sigma^2 Q(\Pi_\sigma, \Pi_\sigma, f,g).
    \]
    Now, the diagonal terms satisfy:
    \[
    Q(\Pi_\sigma, \Pi_\sigma, f,g) = \frac{2}{n^2} \sum_{i<k} |f_{ik} - g_{\sigma(i) \sigma(k)}|^2 = \|f - \sigma g \|^2_{M_n} \geq d_{\M}([f], [g])^2,
    \]
    where the last inequality holds since $d_{\M}([f], [g])^2 = \inf_{\sigma \in S_n} \|f - \sigma g \|^2_{M_n}$. So, we obtain
    \begin{equation}\label{inequalitymain}
    Q(P, f,g) \geq d_{\M}([f], [g])^2 \cdot \sum_{\sigma \in S_n} \lambda_\sigma^2.
    \end{equation}
    Since $\lambda_\sigma \geq 0$ and $\sum_\sigma \lambda_\sigma = 1$, the Cauchy-Schwarz inequality gives
    \[
    1 = \left(\sum_{\sigma \in S_n} \lambda_\sigma \right)^2 \leq |S_n| \cdot \sum_{\sigma \in S_n} \lambda_\sigma^2 = n! \cdot \sum_{\sigma \in S_n} \lambda_\sigma^2,
    \]
    hence:
    \[
    \sum_{\sigma \in S_n} \lambda_\sigma^2 \geq \frac{1}{n!}.
    \]
    Combining this inequality with (\ref{inequalitymain}):
    \[
    Q(P,f,g) \geq \frac{1}{n!} \cdot d_{\M}([f], [g])^2.
    \]
    Since this bound holds for every $P \in \Pi_n$, it holds in particular for the infimum. That is,
    \[
    \GW_2^2(\lb \mathcal{X} \rb, \lb \mathcal{Y} \rb) = \inf_{P \in \Pi_n} Q(P, f,g) \geq \frac{1}{n!} \cdot d_{\M}([f],[g])^2.
    \]
    Taking square roots yields
    \[
    \GW_2(\lb \mathcal{X} \rb, \lb \mathcal{Y} \rb) \geq \frac{1}{\sqrt{n!}} \cdot d_{\M}([f],[g]).
    \]
    We can further refine this bound via Carath\'eodory's Theorem to reduce the number of permutations required in any single decomposition. The Birkhoff polytope of doubly stochastic matrices $\Pi_n \subset \R^{n^2}$ is defined the equality constraints $\sum^n_{j=1} p_{ij} = 1$ for all $i$ and $\sum^n_{i=1} p_{ij} = 1$ for all $j$, along with the nonnegativity constraints $p_{ij} \geq 0$. These $2n$ equality constraints have exactly one redundancy (the sum of all row-sum constraints equal the sum of all column-sum constraints, both equaling $n$), so together they impose $2n-1$ independent linear constraints on $\R^{n^2}$. The affine dimension of $\Pi_n$ is therefore $n^2 - (2n-1) = (n-1)^2$. By Carath\'eodory's Theorem, every point in a convex set of affine dimension $d$ is a convex combination of at most $d + 1$ extremal points. Applying this to $\Pi_n$, then, yields that for every $P \in \Pi_n$, there exist at most $(n-1)^2 + 1$ permutations $\sigma_1, \dots, \sigma_m \in S_n$ with $m \leq (n-1)^2 + 1$ and weights $\mu_1, \dots, \mu_m \geq 0$, $\sum_{k=1}^m \mu_k = 1$, such that
    \[
    P = \sum^m_{k=1} \mu_k \Pi_{\sigma_k}.
    \]
    Applying the Cauchy-Schwarz inequality with this decomposition yields
    \[
    1 = \left( \sum^m_{k=1} \mu_k \right)^2 \leq m \cdot \sum^m_{k=1} \mu_k^2 \leq ((n-1)^2 + 1) \cdot \sum^m_{k=1} \mu_k^2,
    \]
    so $\sum^m_{k=1}\mu_k^2 \geq \frac{1}{(n-1)^2 + 1}$. Hence, for every $P \in \Pi_n$, we have
    \[
    \GW_2(\lb \mathcal{X} \rb, \lb \mathcal{Y} \rb) \geq \frac{1}{\sqrt{(n-1)^2 + 1}} \cdot d_{\M}([f], [g]).\qedhere
    \]
\end{proof}

\subsection{Proof of GW-$k_n$-NN Universal Consistency}\label{sec:gwconsistency}

In this section, we prove \Cref{consistency} by using \Cref{lem:packingmetricdim} to show that $\X$ is separable and has $\sigma$-finite metric dimension. 
Using the map $\Phi_n$ and the $n$-point spaces, we can show that $\X$ is a separable space.
\begin{Lem}\label{separable}
    $\X$ is separable.
\end{Lem}

\begin{proof}
    To show that $\X$ is separable, we exhibit a countable, dense subset. Each element of $\X(n)$ corresponds (after labeling) to a distance matrix $A \in M_n^\leq$ modulo permutation. Since $M_n^\leq$ is a subset of $\R^{n(n-1)/2}$, which is separable, and a quotient by an isometric group action preserves separability, $\X(n)$ is likewise separable for each $n$. Hence, each $\X(n)$ admits a countable dense subset $D_n$, so $D := \cup_{n \in \N} D_n$ is countable and dense in $\cup_{n \in \N} \X(n)$. Because $\X = \cup_{n \in \N} \X(n)$, the countable subset $D$ is dense in $\X$.
\end{proof}

We now prove our main result.

\begin{proof}[Proof of \Cref{consistency}]
Fix $n\in\N$. By \Cref{biLipschitzmap}, there exists a bi-Lipschitz map
\[
\Phi_n:(\X(n),\GW_2)\to(\Mleq,d_{\M}),
\]
so by \Cref{lem:packingmetricdim}, $\dim_{AN}^{s_n}(\X(n), \X)<\infty$ and in particular $\X$ has $\sigma$-finite metric dimension. Because $\X$ is both separable (\Cref{separable}) and has $\sigma$-finite metric dimension, it satisfies the Lebesgue differentiation condition. By Theorem 1 in \citet{NIPS2014_2b764b80}, under the assumptions that $k \to \infty$ and $k/n \to 0$, $k_n$-NN is universally weakly consistent on $(\X, \GW_2)$. If in addition $k/\log(n) \to \infty$, then $k_n$-NN is universally strongly consistent on $(\X, \GW_2)$.
\end{proof}

\section{Universal Consistency of fGW-$k_n$-NN on Node-Attributed Graphs}\label{sec:fgw-consistency}

In this section, we construct an $n$-point attributed space model analogous to the $n$-point model in the Gromov-Wasserstein setting. Using this model, we can show the existence of a bi-Lipschitz map which we then use to establish that $\Y$ has $\sigma$-finite dimension. Since this space is also separable, we therefore prove the universal consistency of $k_n$-NN on $(\Y, \fGW_{\alpha, 1, 2})$.

Throughout this section we fix a feature dimension $d\in\N$ and let $(\Omega,d)=(\R^d,\|\cdot\|_2)$ be the feature space. An \emph{attributed graph} will be modeled as a structured object $(\mathcal{X},f)$ where $\mathcal{X}=(X,d_X,\mu_X)$ is an mm--space with finite support and uniform measure $\mu_X=\frac{1}{|X|}\sum_{x\in X}\delta_x,$
and where $f:X\to\R^d$ is a feature map. As discussed in \Cref{introinfo}, we consider equivalence classes under \emph{weak isomorphisms}. For a collection of weak--isomorphism classes of structured objects, we denote by $\mathbb{Z}_f$ the set of such weak--isomorphism classes and consider the $p,q$-th fused Gromov--Wasserstein space. 
Our set of interest is $\Y$, as defined in \Cref{def:Xatt}. We consider a fixed trade-off parameter $\alpha \in (0,1)$ and $p=1$, $q=2$ in the $\fGW$ distance.

\subsection{A Finite-Dimensional Model for $n$--Point Attributed Spaces}
In this section, we craft an $n$-point attributed space model in the fused Gromov--Wasserstein setting analogous to that in the Gromov--Wasserstein case.
For each $n\in\N$, let $\Y(n)$ denote the set
of weak--isomorphism classes of structured objects such
that $|\supp(\mu_M)|=n$ and $\mu_M$ is the uniform probability measure.
Set
\[
\Y:=\bigcup_{n\in\N} \Y(n).
\]
Define
\[
V_n := M_n\times (\R^d)^n,
\]
where $M_n$ is the space of symmetric $n\times n$ real matrices with zero
 diagonal.
Define the product norm (depending on $\alpha$) as
\begin{equation}\label{eq:alpha-norm}
\|(A,u)\|_{\alpha,n}:= \frac{1}{n^2} \sum_{i,j} \left[ \alpha |A_{ij}|^2 + (1-\alpha)\| u_i\|_2^2 \right]^{1/2}.
\end{equation}

The group $S_n$ acts isometrically on $V_n$ by simultaneously permuting indices in $A$ and
in the vector $u=(u_1,\dots,u_n)$.
Let $\V_n:=V_n/\!\sim$ be the quotient by this action and equip it with the induced metric
\begin{equation}\label{eq:quot-metric}
 d_{\V_n}([(A,u)],[(B,v)]) := \inf_{\sigma\in S_n}\|(A,u)-\sigma(B,v)\|_{\alpha,n}.
\end{equation}
The space $\V_n$ is complete and separable because these properties are preserved under quotient by an isometric group action.

We can also create the fGW analogue of $\Mleq$, which we will denote by $\V_n^\leq$. First, consider $V_n^\leq = M_n^\leq \times(\R^d)^n$, which satisfies the ``triangle inequality'' $A_{ij} + A_{jk} \geq A_{ik}$ and all off-diagonal entries are positive real numbers, and then quotient by the action of $S_n$:
$\V_n^\leq = V^\leq_n / \sim$. 
We then equip this space with the same metric as for $\V_n$:
\[
d_{\V_n}([(A,u)],[(B,v)]) := \inf_{\sigma\in S_n}\|(A,u)-\sigma(B,v)\|_{\alpha,n}.
\]

As with $n$-point spaces, there also exists a bi-Lipschitz map from $\Y(n)$ to $\V_n^\leq$. To show the existence of such a map and in particular to prove that it is well-defined, we first prove the following lemma.
\begin{Lem}\label{lem:finite-uniform-weak-strong}
Let $(\mathcal X,f)$ and $(\mathcal Y,g)$ be weakly isomorphic structured
objects, where
\[
\mu_X=\frac1n\sum_{i=1}^n\delta_{x_i},
\qquad
\mu_Y=\frac1n\sum_{j=1}^n\delta_{y_j},
\]
and $d_X,d_Y$ are genuine metrics on their supports. Then there exists
a permutation $\pi\in S_n$ such that, for all $i,k$,
\[
d_X(x_i,x_k)=d_Y(y_{\pi(i)},y_{\pi(k)})
\quad\text{and}\quad
f(x_i)=g(y_{\pi(i)}).
\]
\end{Lem}
\begin{proof}
Let $(\mathcal W,h)$ and measure-preserving maps
$\phi_X:W\to X$, $\phi_Y:W\to Y$ witness weak isomorphism. Set $W_i=\phi_X^{-1}(\{x_i\})$. Since $(\phi_X)_\#\mu_W=\mu_X$,
we have $\mu_W(W_i)=1/n$. On $W_i\times W_i$, the a.e. identity
\[
d_W(w,w')=d_X(\phi_X(w),\phi_X(w'))=d_X(x_i,x_i)=0
\]
holds. Since $d_W$ is a genuine metric, this implies that $W_i$ is,
up to a null set, a single atom. Thus there exists $w_i\in W_i$ with
$\mu_W(\{w_i\})=1/n$.

The atoms $w_1,\dots,w_n$ exhaust $\mu_W$ up to a null set. Since
$(\phi_Y)_\#\mu_W=\mu_Y$ and each atom of $\mu_Y$ has mass $1/n$,
there is a permutation $\pi\in S_n$ such that
$\phi_Y(w_i)=y_{\pi(i)}$.
Because each pair $\{(w_i,w_k)\}$ has positive $\mu_W\otimes\mu_W$-
measure, the a.e. metric equality holds on every such pair. Hence,
\[
d_X(x_i,x_k)=d_W(w_i,w_k)=d_Y(y_{\pi(i)},y_{\pi(k)}).
\]
Similarly, since each $\{w_i\}$ has positive $\mu_W$-measure, the a.e.
feature equality gives
\[
f(x_i)=h(w_i)=g(y_{\pi(i)}). \qedhere
\]
\end{proof}

We now prove the existence of a bi-Lipschitz map between $\Y(n)$ and $\V_n^\leq$.
\begin{Lem}[A coordinate map for attributed $n$--point spaces]\label{lem:fgw-coordinate-map}
For each $n\in\N$ there is a natural map
\[
\Psi_n:\Y(n)\to\V^\leq_n
\]
which assigns to a representative $(\mathcal{X},f)$ with an ordering
$X=\{x_1,\dots,x_n\}$ the class of
\[
(D_X,\, (f(x_1),\dots,f(x_n)))\in V^\leq_n
\]
where $D_X$ is the $n\times n$ pairwise distance matrix.
The map is well-defined (independent of ordering) and injective. Furthermore, the map is bi-Lipschitz in the following sense: Let $\lb (\mathcal{X}, f) \rb, \lb (\mathcal{Y}, g )\rb \in \Y(n)$ and 
fix $n$. Let $[A,u] := \Psi_n(\lb (\mathcal{X}, f) \rb)$ and $[B,v] := \Psi_n(\lb (\mathcal{Y}, g )\rb)$. Then we have
\[ c_n \fGW_{\alpha, 1,2}(\lb (\mathcal{X}, f) \rb, \lb (\mathcal{Y}, g )\rb) \leq d_{\V_n}([A,u],[B,v] ) \leq C_n \fGW_{\alpha, 1,2}(\lb (\mathcal{X}, f) \rb, \lb (\mathcal{Y}, g )\rb)\]

where $c_n = 1$ globally for all $n$ and $C_n = (n-1)^2+1$ for each $n$. 
\end{Lem}

\begin{proof}
To show that $\Psi_n$ is well-defined, suppose that $(\mathcal{X}, f)$ is weakly isomorphic to $(\mathcal{Y}, g)$. By \Cref{lem:finite-uniform-weak-strong}, there exists $\pi \in S_n$ such that for all $i$ and $k$:
\[d_X(x_i, x_k) = d_Y(y_{\pi(i)}, y_{\pi(k)}) \quad \text{and} \quad f(x_i) = g(y_{\pi(i)}).\]
Hence, $D_X[i,k] = D_Y[\pi(i), \pi(k)]$ and $u_i = v_{\pi(i)}$ for all $i$ and $k$, so
$(D_X, u) = \pi \cdot (D_Y, v)$ under the $S_n$ action. This means that $(D_X, u)$ and $(D_Y, v)$ represent the same class in $\V_n^\leq$:
\[\Psi_n(\lb (\mathcal{X}, f) \rb) = \lb (D_X, u) \rb = \lb (D_Y, v) \rb = \Psi_n(\lb (\mathcal{Y}, g) \rb),\]
thus proving that $\Psi_n$ is well-defined. 

To see that it is injective, let $\Psi_n(\lb (\mathcal{X}, f) \rb) = \Psi_n(\lb (\mathcal{Y}, g) \rb)$. Equality in $\V_n^\leq$ means that there exists $\pi \in S_n$ with $D_X[i,j] = D_Y[\pi(i), \pi(j)]$ for all $i,j$ and $u_i = v_{\pi(i)}$ for all $i$. To show the existence of $(\mathcal{W}, h)$ and maps $\phi_X$ and $\phi_Y$ as required in \Cref{def:strongisom}, take $(\mathcal{W}, h) = (\mathcal{X},f)$ (so, $W = X$, $d_W = d_X$, $\mu_W = \mu_X$, and $h = f$), let $\phi_X = \mathrm{id}_X$, and let $\phi_Y : X \to Y$ be defined by $\phi_Y(x_i) = y_{\pi(i)}$. Then, since $\pi$ is a bijection, we have
\[(\phi_X)_*\mu_W = \mu_X, \quad (\phi_Y)_* \mu_W = \frac{1}{n} \sum_i \delta_{y_{\pi(i)}} = \frac{1}{n} \sum_j \delta_{y_j} = \mu_Y,\]
satisfying the measure-preservation condition. Furthermore, for all $x_i, x_j \in X$:
\[d_X(\phi_X(x_i), \phi_X(x_j)) = d_X(x_i, x_j) = D_X[i,j],\]
\[d_W(x_i, x_j) = d_X(x_i, x_j) = D_X[i,j],\]
and
\[d_Y(\phi_Y(x_i), \phi_Y(x_j)) = d_Y(y_{\pi(i)}, y_{\pi(j)}) = D_Y[\pi(i), \pi(j)] = D_X[i,j].\]
Finally, we have that for all $x_i \in X$:
\[f(\phi_X(x_i)) = f(x_i) = u_i = h(x_i), \quad g(\phi_Y(x_i)) = g(y_{\pi(i)}) = v_{\pi(i)} = u_i = h(x_i).\]
This establishes that $(\mathcal{X}, f)$ and $(\mathcal{Y}, g)$ are weakly isomorphic in the sense of \Cref{def:strongisom}, and so $\Psi_n$ is injective.

To see that the lower bound equals 1, we compare the quotient metric:
\[
d_{\V_n}([A,u], [B,v]) = \inf_{\sigma \in S_n} \| (A,u) - \sigma (B,v)\|_{\alpha, n}
\] 
and the $\fGW_{\alpha, 1,2}$ distance in the discrete case:
\[
\fGW_{\alpha, 1,2}(\lb (\mathcal{X}, f) \rb, \lb (\mathcal{Y}, g) \rb) = \min_{P \in \Pi_n} \frac{1}{n^2} \sum_{i,j,k,l} \left[\alpha |A_{ik} - B_{jl}|^2  + (1-\alpha) \| u_i - v_j \|_{2}^2 \right]^{1/2} P_{ij} P_{kl}.
\] 
Let $(A,u)$ and $(B,v)$ be representatives, and fix any $\sigma \in S_n$. Let $P = P_\sigma$ be a permutation coupling. Then
\[
\fGW_{\alpha, 1,2}(\lb (\mathcal{X}, f) \rb, \lb (\mathcal{Y}, g) \rb) \leq \frac{1}{n^2} \sum_{i,k} \left[ \alpha |A_{ik} - B_{\sigma(i) \sigma(k)}|^2 + (1-\alpha) \| u_i - v_{\sigma(i)} \|_{2} ^2\right]^{1/2}.
\] 
which gives $\fGW_{\alpha, 1,2}(\lb (\mathcal{X}, f) \rb, \lb (\mathcal{Y}, g) \rb) \leq \| (A,u) - \sigma (B,v) \|_{\alpha, n}$. Taking the infimum over $\sigma$ yields the bound $\fGW_{\alpha, 1,2} \leq  d_{\V_n}$.

For the upper bound $C_n$, by the Birkhoff-von Neumann Theorem any $P \in \Pi_n$ admits a convex decomposition of permutation matrices $P = \sum_\sigma \lambda_\sigma \Pi_\sigma$ with $\lambda_\sigma \geq 0$ and $\sum_\sigma \lambda_\sigma = 1$.  Let $(A,u)$ and $(B,v)$ be representatives. Define the fGW objective at coupling $P$ as
\[
Q(P) = \frac{1}{n^2} \sum_{i,j,k,l} \left[ \alpha |A_{ik} - B_{jl} |^2 + (1-\alpha) \| u_i - v_j \|_{2}^2 \right]^{1/2} P_{ij} P_{kl}.
\]
Substituting the Birkhoff-von Neumann decomposition, we have $Q(P) = \sum_{\sigma, \tau} \lambda_\sigma \lambda_\tau Q(\Pi_\sigma, \Pi_\tau)$, where
\[
Q(\Pi_\sigma, \Pi_\tau) = \frac{1}{n^2} \sum_{i,k} \left[\alpha |A_{ik} - B_{\sigma(i) \tau(k)}|^2 + (1-\alpha)\|u_i - v_{\sigma(i)} \|_{2}^2 \right]^{1/2} \geq 0.
\]
Since for all cross terms $Q(\Pi_\sigma, \Pi_\tau) \geq 0$ and for all weights $\lambda_\sigma \lambda_\tau \geq 0$, dropping off-diagonal terms $(\sigma \neq \tau)$ only decreases the sum. Hence, $Q(P) \geq \sum_\sigma \lambda_\sigma^2 Q(\Pi_\sigma, \Pi_\sigma)$.
We next need to bound each diagonal term from below. Let $a_{ik}^\sigma := |A_{ik} - B_{\sigma(i) \sigma(k)}|$ and $b_i^\sigma := \|u_i - v_{\sigma(i)} \|_2$. Expanding $Q(\Pi_\sigma, \Pi_\sigma)$ gives
\[
Q(\Pi_\sigma, \Pi_\sigma) = \frac{1}{n^2} \sum_{i,k} \left[ \alpha (a^\sigma_{ik})^2  + (1-\alpha) (b_i^\sigma)^2 \right]^{1/2} = \|(A,u) - \sigma(B,v)\|_{\alpha, n} \geq d_{\V_n}([A,u], [B,v]),
\]
where the last inequality holds since $d_{\V_n}([A,u], [B,v]) = \inf_\sigma \| (A,u) - \sigma (B,v)\|_{\alpha, n}$. Combining these results gives $Q(P) \geq d_{\V_n}([A,u], [B,v]) \cdot \sum_\sigma \lambda^2_\sigma$. By the same Cauchy-Schwarz argument combined with Carath\'eodory's theorem as in the proof of \Cref{biLipschitzmap}, with at most $(n-1)^2 + 1$ permutations in the support of the decomposition, we obtain
\[\sum_\sigma \lambda_\sigma^2 \geq \frac{1}{(n-1)^2 + 1},\]
giving
\[
Q(P) \geq \frac{d_{\V_n}([A,u], [B,v])}{(n-1)^2 + 1} \quad \forall P \in \Pi_n.
\] 
Taking the infimum over $P$ yields
\[
\fGW_{\alpha, 1,2}(\lb (\mathcal{X}, f) \rb, \lb (\mathcal{Y}, g )\rb ) \geq \frac{1}{(n-1)^2 + 1} d_{\V_n}([A,u], [B,v]).\qedhere
\] 
\end{proof}

\subsection{Proof of fGW-$k_n$-NN Universal Consistency}

To prove universal consistency of the $k_n$-NN classifier in $(\Y, \fGW_{\alpha, 1, 2})$, it remains to show that $\Y$ is separable and has $\sigma$-finite dimension.
The bi-Lipschitz equivalence from \Cref{lem:fgw-coordinate-map} proves the separability of $\Y$ analogously to $\X$, since we can identify each element of the union $\Y = \cup_{n \in \N} \Y(n)$ with each $\V^\leq_n$.

\begin{Lem}\label{Xattseparable}
    $\Y$ is separable.
\end{Lem}

\begin{proof}
    Because each element of $\Y(n)$ corresponds after labeling to an attributed object $(A,u) \in V_n^\leq$, separability is preserved under a quotient by an isometric group action, so $\Y(n)$ is separable for each $n$. Take a countable dense subset $F_n$; then $F = \cup_{n \in \N} F_n$ is countable and dense in $\cup_{n\in \N} \Y(n)$, so $F$ is dense in $\Y$.
\end{proof}

We now prove that $k_n$-NN is universally consistent on $(\Y, \fGW_{\alpha, 1, 2})$.

\begin{proof}[Proof of \Cref{consistencyFGW}]
    Let $\Psi_n : \Y(n) \to \V^\leq_n$ be the bi-Lipschitz map defined by \Cref{lem:fgw-coordinate-map}. Note that the product norm defined in \Cref{eq:alpha-norm} is not the standard Euclidean norm. However, since all norms are equivalent on finite-dimensional spaces, \Cref{lem:packingmetricdim} still applies (by which we mean that we can prove \Cref{lem:packingmetricdim} with the norm in \Cref{eq:alpha-norm} following the same method; while this would yield a different bounding constant on the number of cubes, this constant is still finite). Hence, due to the existence of the bi-Lipschitz map $\Psi_n$,  $\dim_{AN}^{s_n}(\Y(n), \Y) < \infty$ and in particular $\Y$ has $\sigma$-finite metric dimension by \Cref{lem:packingmetricdim}. The space $\Y$ is also separable (by \Cref{Xattseparable}). Hence, by Theorem 1 in \citet{NIPS2014_2b764b80},  under the assumptions that $k \to \infty$ and $k/n \to 0$, $k_n$-NN is universally weakly consistent on $(\Y, \fGW_{\alpha, 1, 2})$. If in addition $k/\log(n) \to \infty$, then $k_n$-NN is universally strongly consistent on $(\Y, \fGW_{\alpha, 1, 2})$.
\end{proof}

\section{Numerical Experiments}\label{sec:numericalresults}
We present the results of our numerical experiments with GW-$k_n$-NN and fGW-$k_n$-NN on several categories of graph datasets: three molecular, two social, and two synthetic (random graph) datasets. The summary statistics for all datasets are shown in Table \ref{table:sumstats}. All experiments were implemented on an Apple MacBook Pro with an Apple M2 chip and 16~GB RAM. For reproducibility, all experiments were run in Jupyter notebooks, and the code that produced the results in this section, along with the results themselves, can be found on our project GitHub.\footnote{\texttt{https://github.com/khohmeier/gw-knn}} Parameter tuning for $\alpha$ in fGW-$k_n$-NN was conducted via a symmetric linear search in the interval $[0,1]$. In all experiments, training and testing sets were obtained via 10-fold cross validation.

To compare against our GW-$k_n$-NN and fGW-$k_n$-NN methods, we also implement a custom-built graph convolutional network and several graph kernel methods whose code is derived, with modifications, from the implementations in \citet{vay2019fgw}; these modifications include computing additional metrics such as F1 scores and by-class classification accuracy. For the GCN, Bayesian hyperparameter tuning was used to determine number of epochs, batch size, and learning rate for the neural network. For number of epochs, a range from 20 to 100, with a step size of 10, was tested; for the learning rate, we used a range from $1e^{-4}$ to $1e^{-1}$, where the learning rate value was sampled via a log scale; and finally, batch sizes of 16, 32, and 64 were tested. For the attributed graphs, we used the graph kernel methods of HOPPER \cite{hopperk}, the propagation kernel \cite{propak}, and the indefinite kernel matrix $e^{-\gamma FGW}$ \cite{vay2019fgw}, and classification is performed using an SVM. These methods will be denoted as HOPPER(V), PROPA(V), and fGW(V), respectively. For the SVM when applied to HOPPER(V) and PROPA(V), the parameter $C$ was cross-validated via logarithmically spaced values between $10^{-5}$ and $10^5$, for a total of 15 values. The $t_{max}$ parameter for PROPA(V) was chosen from a discrete set of 7 values: $\{1, 3, 5, 8, 10, 15, 20\}$. To tune $\gamma$ and $\alpha$ in fGW(V), $\gamma$ is cross-validated from 15 exponentially spaced values from $2^{-10}$ to $2^{10}$, while $\alpha$ is cross-validated from a search space that includes the endpoint extremes of {0, 1}, log-spaced values near 0, linear mid-range values, and log-spaced values near 1. For fGW(V), the SVM $C$ parameter ranges from $10^{-4}$ to $10^4$ and again uses logarithmically spaced values. To compute the feature distance matrix in both fGW-$k_n$-NN and fGW(V), we use the $\ell_2$ distance between the real-valued vector attributes of the nodes. For the non-attributed graphs, we used the shortest path kernel \cite{spk} and the graphlet count kernel \cite{gk}, again using an SVM to classify. The parameter search space in this case involves the SVM parameter $C$, again cross-validated via logarithmically spaced values between $10^{-5}$ and $10^5$; for GK(V), we use graphlet sizes of $3,4,$ and $5$ and test $\delta$ and $\epsilon$ parameters of $0.1$ and $0.05$. All graph kernel methods utilized nested 10-fold cross validation.

\begin{table}[H]
    \centering
    \scalebox{0.8}{\begin{tabular}{|l|l|l|l|l|}
    \hline
    Dataset & Size & Avg. Nodes & Classes & Node Features? \\
    \hline
    BZR     & 405  & 36         & 2   & Yes            \\
    \hline
    COX-2   & 467  & 41         & 2   & Yes            \\
    \hline
    PROTEINS & 1113 & 39 & 2 & Yes
    \\
    \hline
    ERSBM1  & 200  & 100        & 2       & No             \\
    \hline
    ERSBM2  & 200  & 100        & 2       & No             \\
    \hline
    IMDb-B  & 1000 & 20         & 2       & No             \\
    \hline
    IMDb-M  & 1500 & 13         & 3       & No           \\
    \hline
    \end{tabular}}
    \caption{Properties of molecular, social, and synthetic graphs used in all numerical experiments.}
    \label{table:sumstats}
\end{table}

\subsection{Datasets} 

For our numerical experiments, we utilize 7 datasets. Of these, 5 are popular benchmark datasets: BZR, COX-2, PROTEINS, IMDb-B, and IMDb-M. BZR, COX-2, and PROTEINS consist of vector-attributed graphs, while IMDb-B and IMDb-M are non-attributed graphs. The remaining two datasets are synthetic random graph datasets, which we will denote by ERSBM1 and ERSBM2.

The two small molecule datasets, BZR and COX-2, are both derived from \citet{bzr}. BZR is a dataset of benzodiazepine receptor ligands, which are stored and treated as rotation-invariant graphs. COX-2 consists of cyclooxygenase-2 inhibitors. The binary classes for both datasets are active, or binding (class 1) versus inactive, or non-binding (class -1). For BZR, the number of graphs labeled as class 1 is 86, while the number of elements in class -1 is 319. For COX-2, the number of graphs labeled as class 1 is 102, and the number of graphs in class -1 is 365. Hence, both are unbalanced classification problems. For both datasets, the three-dimensional vector attributes assigned to each node encode the three-dimensional structure of each molecule, and the entries of each vector denote the coordinates of the respective atom in $\R^3$; see section 3 in \citet{mahe:hal-00433580}.  The classification task for both datasets will involve binary classification, active versus inactive.
The dataset PROTEINS contains macromolecules and was curated by \citet{DOBSON2003771}, where the classification task is again binary and involves predicting whether a given protein is an enzyme (class 1) or a non-enzyme (class 2). This likewise is an unbalanced classification problem, with 663 enzymes and 450 non-enzymes. Node features are 29-dimensional vectors and encode feature information such as amino acid frequencies \cite{DOBSON2003771}. The conversion of each protein into a graph structure was accomplished via the strategy proposed by \citet{bioinformatics}, where nodes denote secondary structure elements within each protein, such as helices, sheets, and turns, and edges between nodes exist in one of the following situations: neighbors along the amino acid sequence, neighbors in space within the protein
structure, or three nearest spatial neighbors.

For the non-attributed graphs, the social datasets consist of two IMDb movie collaboration networks curated by \cite{imdbdata}: IMDb-B (binary) with 500 graphs per class, and IMDb-M (multiclass) with 500 graphs per class. In each graph, nodes represent an actor or actress, edges between them indicate that the actors appeared in another movie together, and the entire ego-network represents the cast of a single film. Each network then receives a class label based on film genre: in the binary dataset, the two movie genres are Action (Class 0) and Romance (Class 1), while the multiclass data contains the genres Comedy (Class 1), Romance (Class 2), and Sci-Fi (Class 3). The classification task is to determine which genre each ego-network falls into. Finally, the two synthetic datasets, ERSBM1 and ERSBM2, consist of Erd\H{o}s-Renyi and stochastic block model graphs. Both datasets contain 100 graphs from each random graph model; all graphs have the same number of nodes; and all the stochastic block model graphs have two equal-sized communities. They differ in edge addition parameters: for ERSBM1, the Erd\H{o}s-Renyi graphs have an edge addition parameter of 0.1, and the stochastic block model graphs have a within-community edge addition probability of 0.15 and between-community probability of 0.05, while for ERSBM2, these parameters are 0.25 for Erd\H{o}s-Renyi and 0.45 and 0.05 for stochastic block model. The classification task is to predict whether a graph falls into the Erd\H{o}s-Renyi or stochastic block random graph model.

\subsection{Results}

\Cref{tab:allattributedresults} illustrates the results of numerical experiments with the attributed graph datasets, reporting average accuracy across all 10 folds of cross validation. For all three datasets, fGW-$k_n$-NN is the clear winner in terms of accuracy, with GW-$k_n$-NN also performing well. However, since these classification problems are unbalanced, comparing models via other metrics such as F1 scores and by-class classification accuracies offers a more complete picture of the performance of these methods. These additional metrics, as well as runtime, are shown in \Cref{tab:attribF1} in Appendix A. It should be noted that although GW-$k_n$-NN performed worse on COX-2 than GCN and fGW(V), GCN did not correctly classify any class 1 (active) molecules, while fGW(V) had a long runtime. The methods HOPPER(V) and PROPA(V) also failed to predict any active molecules in the COX-2 dataset correctly. Hence, GW-$k_n$-NN yields comparable accuracy to other methods, while offering efficiency improvements and better differentiation between classes. Likewise, fGW-$k_n$-NN also offers stronger by-class differentiation versus other methods. The overall accuracy improvements offered by fGW-$k_n$-NN versus GW-$k_n$-NN suggest that including node feature information for attributed graphs provides additional important information when differentiating graphs of each class.

\begin{table}[H]
    \centering
     \scalebox{0.8}{\begin{tabular}{|c|c|c|c|}
    \hline
        Model & BZR & COX-2 & PROTEINS \\
        \hline

         \begin{tabular}{c}
         fGW-$k_n$-NN     
        \end{tabular}
         & \begin{tabular}{c}
           $\boldsymbol{86.95\% \pm 6.14\%}$ \\
           $k=1, \ \alpha=0.9$
        \end{tabular} & \begin{tabular}{c}
          $\boldsymbol{80.31\% \pm 2.56\%}$  \\
          $k=7, \ \alpha=0.0$
        \end{tabular} & \begin{tabular}{c}
           $\boldsymbol{74.75\% \pm 3.70\%}$ \\
           $k=10, \ \alpha=0.0$
        \end{tabular} \\
         \hline
       
         \begin{tabular}{c}
           GW-$k_n$-NN 
         \end{tabular}
        & \begin{tabular}{c}
           $\mathit{85.21\% \pm 6.64\%}$ \\
           $k=1$  
        \end{tabular} & \begin{tabular}{c}
          $77.95\% \pm 3.29\%$    \\
          $k=4$ 
        \end{tabular} & \begin{tabular}{c}
           $70.44\% \pm 3.89\%$ \\
           $k=7$  
        \end{tabular} \\
        \hline

        \begin{tabular}{c} 
        GCN
        \end{tabular} & \begin{tabular}{c}
           $82.49\% \pm 4.43\%$   \\
           $(e=70, \ b=16, $\\
           $r=0.0079)$
        \end{tabular} & \begin{tabular}{c}
          $78.16\% \pm .81\%$   \\
          $(e=50, \ b=16,$ \\
          $r=0.071)$
        \end{tabular}  & \begin{tabular}{c}
          $59.57\% \pm .17\%$   \\
          $(e=50, \ b=16,$ \\
          $r=0.071)$
        \end{tabular}  \\
        \hline
        
        fGW(V) & \begin{tabular}{c}
           $84.88\% \pm 5.09\%$   
        \end{tabular} & \begin{tabular}{c}
          $\mathit{78.51\% \pm 4.41\%}$    
        \end{tabular} &  \begin{tabular}{c}
          $73.21\% \pm 2.88\%$    
        \end{tabular} \\
        \hline

        HOPPER(V) & \begin{tabular}{c}
           $82.44\% \pm 6.53\%$   
        \end{tabular} & \begin{tabular}{c}
           $77.23\% \pm 5.39\%$   
        \end{tabular} &  \begin{tabular}{c}
           $\mathit{74.73 \% \pm 3.85\%}$   
        \end{tabular} \\
        \hline

        PROPA(V) & \begin{tabular}{c}
           $80.24\% \pm 6.85\%$   
        \end{tabular} & \begin{tabular}{c}
           $77.45\% \pm 5.48\%$   
        \end{tabular} &  \begin{tabular}{c}
           $70.45\% \pm 2.86\%$   
        \end{tabular} \\
        \hline   
        
        \end{tabular}}
    \caption{Vector attributed graph results, binary classification, best results for all models. Reported values are averages across all folds of cross validation. Ten-fold CV was implemented for all models except for fGW(V), which uses nested 10-fold CV. Bolded entries are the best-performing models, while italicized entries denote the second-best models.}
    \label{tab:allattributedresults}
\end{table}

In \Cref{tab:allnonattrib}, we show the experimental results for our non-attributed graph datasets, with GW-$k_n$-NN performing the best on two out of four datasets. For the two random graph datasets, GW-$k_n$-NN clearly performed well on ERSBM2, but had much weaker performance on ERSBM1. This difference is likely due to the structural differences between the graphs in each dataset: since GW-$k_n$-NN is a structure-based method, the similarity of the two classes in ERSBM1 led to difficulty differentiating between the classes when using GW-$k_n$-NN, and the structural similarity between the graph classes also likely explains why the graph kernel methods outperformed GW-$k_n$-NN in this case. 
For ERSBM2, while SPK(V) was highly efficient and close in accuracy to GW-$k_n$-NN, GW-$k_n$-NN classified every graph correctly. However, SPK(V)'s performance did not persist when applied to the IMDb-B dataset, where GW-$k_n$-NN strongly outperformed. Interestingly, using adjacency instead of pairwise distances in this dataset also improved the performance of GW-$k_n$-NN, but this came at the cost of significantly increased runtimes. For IMDb-M, GK(V) offered the best results, but this came at a total runtime of nearly 13.5 hours, compared to around 11 minutes for GW-$k_n$-NN. The additional metrics of F1 scores and runtimes are provided via \Cref{tab:nonattribF1} in Appendix A.

\begin{table}[H]
    \centering
    \scalebox{0.8}{\begin{tabular}{|c|c|c|c|c|}
    \hline
        Model & ERSBM1 & ERSBM2 & IMDb-B & IMDb-M \\
        \hline
        \begin{tabular}{c}
           GCN   
        \end{tabular}
        & \begin{tabular}{c}
           $54.50\% \pm 7.57\%$   \\
           $(e=100, \ r=0.028,$
           \\ $b=64)$ 
        \end{tabular} & \begin{tabular}{c}
           $60.00\% \pm 18.30\%$   \\
           $e=70, \ r=0.00026,$
           \\ $b=64)$ 
        \end{tabular} & \begin{tabular}{c}
           $57.80\% \pm 4.02\%$   \\
           $(e=80, \ r=0.0011,$\\
           $ b=32)$ 
        \end{tabular} & \begin{tabular}{c}
           $37.37\% \pm 3.05\%$   \\
            $(e=70, \ r=0.0014,$\\
            $b=16)$  
        \end{tabular} \\
        \hline
        
        \begin{tabular}{c}
          GW-$k_n$-NN    
        \end{tabular}
         & \begin{tabular}{c}
           $62.00\% \pm 7.81\%$   \\
           $(k=3)$
        \end{tabular} &\begin{tabular}{c}
           $\boldsymbol{100.00\% \pm 0.00\%}$  \\
           (all $k$)
        \end{tabular} & \begin{tabular}{c}
           $\mathit{65.10\% \pm 5.32\%}$ \\
           $(k=1)$
        \end{tabular} & \begin{tabular}{c}
           $38.47\% \pm 3.65\%$  \\
           $(k=8)$
        \end{tabular} \\
         \hline
         
         \begin{tabular}{c}
           GW-$k_n$-NN(adj)  
         \end{tabular}
           & \begin{tabular}{c}
           $55.00\% \pm 14.14\%$ \\
           $(k=3)$ 
        \end{tabular} & \begin{tabular}{c}
           $99.00\% \pm 2.00\%$   \\
            $(k=7)$ 
        \end{tabular} & \begin{tabular}{c}
           $\boldsymbol{68.30\% \pm 4.00\%}$  \\
           $(k=6)$ 
        \end{tabular} & \begin{tabular}{c}
           $\mathit{38.73\% \pm 4.56\%}$  \\
           $(k=8)$ 
        \end{tabular} \\
        \hline
        
        GK(V) & \begin{tabular}{c}
           $\mathit{76.00\% \pm 9.43\%}$   
        \end{tabular} & \begin{tabular}{c}
           $99.00\% \pm 2.00\%$   
        \end{tabular} & \begin{tabular}{c}
           $64.10\% \pm 3.05\%$   
        \end{tabular}  & \begin{tabular}{c}
           $\boldsymbol{40.33\% \pm 3.99\%}$   
        \end{tabular} \\
        \hline

        SPK(V) & \begin{tabular}{c}
           $\boldsymbol{96.50\% \pm 7.09\%}$   
        \end{tabular} & \begin{tabular}{c}
           $\mathit{99.50\% \pm 1.50\%}$   
        \end{tabular} & \begin{tabular}{c}
           $54.60\% \pm 3.56\%$   
        \end{tabular} & \begin{tabular}{c}
           $38.20\% \pm 3.18\%$   
        \end{tabular}  \\
        \hline
    \end{tabular}}
    \caption{Non-attributed graph results, best results for all models. Reported values are averages across all folds of cross validation. Ten-fold CV was implemented for all models. Bolded entries are the best-performing models, while italicized entries denote the second-best models.}
    \label{tab:allnonattrib}
\end{table}

\section{Conclusions}

In this paper, we introduced the idea of applying the $k_n$-NN classifier to graph classification problems using the Gromov--Wasserstein and fused Gromov--Wasserstein distances, establishing that GW-$k_n$-NN is universally consistent on the space of graphs and that fGW-$k_n$-NN is universally consistent on the space of node-attributed graphs. In practice, both GW-$k_n$-NN and fGW-$k_n$-NN perform well across many types of graph datasets and can even outperform other well-established graph classification methods. However, efficiency remains a key issue, especially for fGW-$k_n$-NN due to the need to tune the $\alpha$ parameter. The main area of future work focuses on the question: How far can we generalize $k_n$-NN consistency results? Some potential areas for future work are the following.

\begin{itemize}
    \item We established universal consistency for  $(\X, \GW_2)$ and $(\Y, \fGW_{\alpha, 1, 2})$. For what other values of $p$ (and $p$ and $q$ in the fGW case) will this result hold? For what other spaces of equivalence classes of mm-spaces and structured objects will universal consistency hold?
    \item  Analysis of $k_n$-NN universal consistency in the general Gromov--Wasserstein and the general fused Gromov--Wasserstein settings still remain as open problems. We conjecture that, as in the Wasserstein case, universal consistency of $k_n$-NN will not hold in full generality. Hence, some restrictions will be required on the spaces of interest.
    \item Rather than trying to prove universal consistency results separately for each GW framework, taking advantage of the Z-Gromov--Wasserstein framework would offer a strategy to prove universal consistency in the most generality possible. Hence, one might consider: What is the most general choice of $Z$ that guarantees universal consistency?
\end{itemize}

\bibliographystyle{plainnat}
\bibliography{biblio.bib}

\section*{Appendix}

Due to the unbalanced nature of some of the classification problems, we also present F1 scores for each model as an indication of how well each method predicted each class. These results are shown in \Cref{tab:attribF1}, along with runtimes for each model.

\begin{table}[H]
    \centering
     \scalebox{0.8}{\begin{tabular}{|c|c|c|c|}
    \hline
        Model & BZR & COX-2 & PROTEINS \\
        \hline

         \begin{tabular}{c}
         fGW-$k_n$-NN     
        \end{tabular}
         & \begin{tabular}{c}
           $86.61\% \pm 6.56\%$ \\
           23.02 min
        \end{tabular} & \begin{tabular}{c}
          $75.56\% \pm 2.82\%$ \\ 
          26.14 min
        \end{tabular} & \begin{tabular}{c}
           $74.16\% \pm 3.83\%$ \\
          3.00 hrs
        \end{tabular} \\
         \hline
       
         \begin{tabular}{c}
           GW-$k_n$-NN 
         \end{tabular}
        & \begin{tabular}{c}
           $84.36\% \pm 7.18\%$\\
           3.93 min
        \end{tabular} & \begin{tabular}{c}
          $74.67\% \pm 3.62\%$ \\
          4.09 min
        \end{tabular} & \begin{tabular}{c}
           $69.68\% \pm 4.01\% $ \\
           30.20 min
        \end{tabular} \\
        \hline

        \begin{tabular}{c} 
        GCN
        \end{tabular} & \begin{tabular}{c}
           $79.20\% \pm5.94\%$ \\
           24.18 min
        \end{tabular} & \begin{tabular}{c}
          $68.58\% \pm 1.10\%$ \\
          28.37 min 
        \end{tabular}  & \begin{tabular}{c}
          $44.47 \% \pm .20\%$   \\
          55.09 min
        \end{tabular}  \\
        \hline
        
        fGW(V) & \begin{tabular}{c}
           $83.82\% \pm 6.13\%$ \\
           35.05 min
        \end{tabular} & \begin{tabular}{c}
          $74.01\% \pm 7.64\%$ \\ 
          45.77 min
        \end{tabular} &  \begin{tabular}{c}
          $72.53\% \pm 3.34\%$ \\ 
          6.84 hrs
        \end{tabular} \\
        \hline

        HOPPER(V) & \begin{tabular}{c}
           $80.83\% \pm 7.59\%$ \\
           10.28 min
        \end{tabular} & \begin{tabular}{c}
           $67.61\% \pm 7.80\%$  \\
           25.97 min
        \end{tabular} & \begin{tabular}{c}
           $74.09\%  \pm 4.41\%$  \\
           30.41 min
        \end{tabular}  \\
        \hline

        PROPA(V) & \begin{tabular}{c}
           $71.92\% \pm 10.55\%$ \\
           0.93 min
        \end{tabular} & \begin{tabular}{c}
           $67.71\% \pm 7.86\%$  \\
           1.40 min
        \end{tabular} &  \begin{tabular}{c}
           $69.21\% \pm 3.59\%$  \\
           12.26 min
        \end{tabular} \\
        \hline    \end{tabular}}
    \caption{Vector attributed graph results, F1 scores and runtimes.}
    \label{tab:attribF1}
\end{table}

In \Cref{tab:nonattribF1}, we present the F1 scores and runtimes for our non-attributed graph results. While all of these datasets were balanced classification problems, the F1 scores still indicate the ability of the different methods to predict each class.

\begin{table}[H]
    \centering
    \scalebox{0.8}{\begin{tabular}{|c|c|c|c|c|}
    \hline
        Model & ERSBM1 & ERSBM2 & IMDb-B & IMDb-M \\
        \hline
        \begin{tabular}{c}
           GCN   
        \end{tabular}
        & \begin{tabular}{c}
           $41.18\% \pm 11.96\%$  \\
           20.79 min
        \end{tabular} & \begin{tabular}{c}
           $48.20\% \pm 25.55\%$ \\
           19.85 min
        \end{tabular} & \begin{tabular}{c}
           $57.53\% \pm 4.03\%$   \\
           52.53 min 
        \end{tabular} & \begin{tabular}{c}
           $31.19\% \pm 4.04\%$   \\
            58.21 min  
        \end{tabular} \\
        \hline
        
        \begin{tabular}{c}
          GW-$k_n$-NN    
        \end{tabular}
         & \begin{tabular}{c}
           $57.42\% \pm 9.38\%$ \\
           2.32 min
        \end{tabular} &\begin{tabular}{c}
           $100.00\% \pm 0.00\%$  \\
           2.83 min
        \end{tabular} & \begin{tabular}{c}
           $64.87\% \pm 5.34\%$ \\
           7.69 min
        \end{tabular} & \begin{tabular}{c}
           $35.22\% \pm 3.47\%$  \\
           10.82 min
        \end{tabular} \\
         \hline
         
         \begin{tabular}{c}
           GW-$k_n$-NN(adj)  
         \end{tabular}
           & \begin{tabular}{c}
            $53.75\% \pm 14.17\%$ \\
            27.72 min
        \end{tabular} & \begin{tabular}{c}
           $99.00\% \pm 2.01\%$   \\
            31.54 min 
        \end{tabular} & \begin{tabular}{c}
           $68.15\% \pm 4.05\%$ \\
           7.39 hrs 
        \end{tabular} & \begin{tabular}{c}
           $35.94\% \pm 4.34\%$  \\
           6.53 hrs 
        \end{tabular} \\
        \hline
        
        GK(V) & \begin{tabular}{c}
           $76.04\% \pm 9.81\%$   \\
           4.23 min
        \end{tabular} & \begin{tabular}{c}
           $98.99\% \pm 2.13\%$   \\
           7.03 min
        \end{tabular} & \begin{tabular}{c}
           $63.89\% \pm 3.43\%$  \\
           3.23 hrs
        \end{tabular}  & \begin{tabular}{c}
           $37.39\% \pm 4.03\%$  \\
           13.45 hrs
           
        \end{tabular} \\
        \hline

        SPK(V) & \begin{tabular}{c}
            $96.52\% \pm 7.44\%$   \\
            2.58 sec 
        \end{tabular} & \begin{tabular}{c}
           $99.50\% \pm 1.58\%$ \\
           3.63 sec
        \end{tabular} & \begin{tabular}{c}
           $53.42\% \pm 4.26\%$   \\
           1.42 min 
        \end{tabular} & \begin{tabular}{c}
           $30.90\% \pm 3.10\%$  \\
           4.48 min
        \end{tabular}  \\
        \hline
    \end{tabular}}
    \caption{Non-attributed graph results, F1 scores and runtimes.}
    \label{tab:nonattribF1}
\end{table}

\end{document}